\documentclass[sigconf]{acmart}

\AtBeginDocument{%
  }

\setcopyright{acmlicensed}
\copyrightyear{2026}
\acmYear{2026}
\acmDOI{}
\acmConference[Conference'26]{Conference}{June 03--05, 2018}{Woodstock, NY}
\acmISBN{978-1-4503-XXXX-X/2018/06}

\usepackage{amsmath}
\usepackage{graphicx}
\usepackage{hyperref}
\usepackage{booktabs, tablefootnote}
\usepackage{multirow}
\usepackage{subcaption}
\usepackage{lipsum} 
\usepackage{bbm}
\usepackage{enumitem,multicol,float,adjustbox}
\usepackage{threeparttable}
\usepackage{mathtools}
\usepackage{xcolor, colortbl}
\usepackage{makecell}
\usepackage{dblfloatfix}
\usepackage{caption}
\usepackage{listings}
\usepackage{array}
\usepackage{enumitem}
\usepackage{xspace}
\usepackage[ruled,vlined,linesnumbered]{algorithm2e}

\renewcommand{\arraystretch}{1.1}

\begin{document}
\copyrightyear{2026}
\acmYear{2026}
\setcopyright{cc}
\setcctype{by}
\acmConference[CIKM '26]{Proceedings of the 35th ACM International Conference on Information and Knowledge Management}{November 07--11, 2026}{Rome, Italy}
\acmBooktitle{Proceedings of the 35th ACM International Conference on Information and Knowledge Management (CIKM '26), November 07--11, 2026, Rome, Italy}
\acmDOI{10.1145/3799682.3841033}
\acmISBN{979-8-4007-2539-5/2026/11}
\title{Beyond Observed Auxiliary Relations: Environment-Conditioned Modeling for Multi-Behavior Recommendation}

\author{Seunghan Lee}
\affiliation{%
  \institution{Korea University}
  \city{Seoul}
  \country{Republic of Korea}}
\email{seunghanlee@korea.ac.kr}

\author{Hyunsik Yoo}
\affiliation{%
  \institution{University of Illinois Urbana-Champaign}
  \city{Champaign}
  \state{IL}
  \country{USA}}
\email{hy40@illinois.edu}

\author{Jian Kang}
\affiliation{%
  \institution{MBZUAI}
  \city{Abu Dhabi}
  \country{United Arab Emirates}}
\email{jian.kang@mbzuai.ac.ae}

\author{Susik Yoon}
\affiliation{%
  \institution{Korea University}
  \city{Seoul}
  \country{Republic of Korea}}
\email{susik@korea.ac.kr}

\author{SeongKu Kang}
\authornote{Corresponding author}
\affiliation{%
  \institution{Korea University}
  \city{Seoul}
  \country{Republic of Korea}}
\email{seongkukang@korea.ac.kr}

\renewcommand{\shortauthors}{Seunghan Lee, Hyunsik Yoo, Jian Kang, Susik Yoon, and SeongKu Kang}

\begin{abstract}
Multi-behavior recommendation (MBR) leverages auxiliary behavioral signals, such as clicks and add-to-cart, to enhance target behavior prediction like purchases.
While recent graph neural network–based approaches have achieved strong performance by systematically propagating auxiliary behavior signals, they still suffer from two fundamental challenges inherent to auxiliary behaviors:
(1) missing auxiliary signals, which hinder generalization to items without auxiliary observations, and
(2) unreliable auxiliary signals, which amplify noise misaligned with the target behavior.
To address these challenges in a unified manner, we propose \proposed, an environment-conditioned MBR framework that addresses missing and unreliable auxiliary signals through two complementary modules conditioned on auxiliary observability.
Extensive experiments demonstrate that \proposed consistently outperforms state-of-the-art baselines, achieving up to 7.82\% gains in HR@10 overall and up to 44.2\% gains for target items without auxiliary observations, highlighting its ability to capture hidden preferences beyond observed auxiliary relations.
Our code is available at: \url{https://github.com/LSH0411/BOAR}.
\end{abstract}


\newcommand{\proposed}{\textsc{BOAR}\xspace}
\newcommand{\smallsection}[1]{{\vspace{0.03in} \noindent \bf {#1}}}

\maketitle

\section{Introduction}

Users in modern recommender systems rarely express their preferences through a single type of behavior. 
Before making a purchase, users typically engage in multiple auxiliary behaviors, such as clicking items, adding them to carts, or saving them to wishlists. 
These auxiliary behaviors provide indirect yet complementary signals that reflect users’ latent preferences with respect to the final target behavior, such as purchasing. 
Accordingly, multi-behavior recommendation (MBR) aims to jointly model auxiliary behaviors together with the target behavior to better capture user preferences and improve recommendation accuracy~\cite{MBGCN,GHCF}.

In multi-behavior recommendation, a wide range of approaches has been explored, from traditional matrix factorization~\cite{CMF, MF-BPR} to deep neural network-based models~\cite{NMTR}. 
Among these approaches, graph neural networks (GNNs) have recently emerged as a core paradigm by explicitly modeling the relational structure between users and items, rather than treating interactions independently~\cite{NGCF}.
By representing different behavior types as distinct relations on graphs, GNNs systematically propagate auxiliary behavior signals to enhance target behavior prediction~\cite{MBGCN}.
This relational modeling is particularly effective in real-world settings, where target behaviors are extremely sparse.
Furthermore, combined with advanced learning paradigms such as self-supervised and multi-task learning, recent methods learn richer and more generalizable representations and achieve state-of-the-art performance~\cite{S-MBRec,MBSSL,COPF}.

\begin{figure*}[t]
    \centering
    \includegraphics[width=\textwidth]{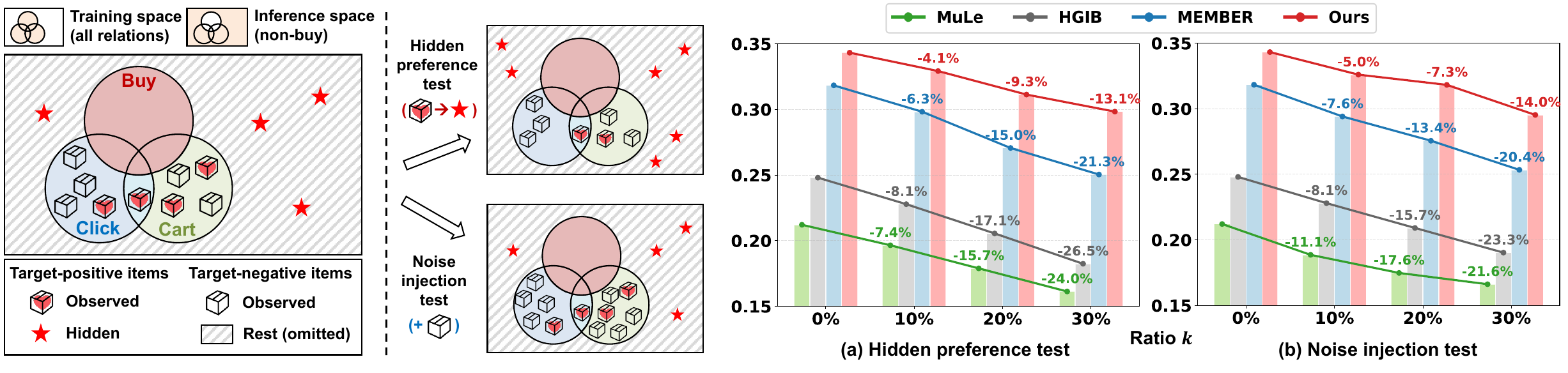}
    \caption{Overview of MBR setup.
    The left panel illustrates training and inference spaces.
    The right panel presents an empirical evaluation with two settings, reporting HR@10:
    (a) Hidden preference test, removing auxiliary interactions from observed target-positive items; and
    (b) Noise injection test, adding auxiliary interactions to target-negative items.
    Target-positive and target-negative items are defined with respect to the target behavior.
    Detailed setups are provided in Section~\ref{subsub:robust}.
    }
    \label{fig:figure1}
\end{figure*}

\label{sec:1}
Although effective, existing GNN-based methods still face two fundamental challenges arising from the nature of auxiliary behaviors.
Specifically, auxiliary signals are often missing or unreliable: their absence does not necessarily indicate the absence of user interest, while their presence does not necessarily indicate~target~intent.
\begin{itemize}[leftmargin=*, labelsep=0.5em, topsep=2pt, itemsep=0pt]
    \item \textbf{C1: Missingness of Auxiliary Signals.}
    In practice, users do not necessarily exhibit auxiliary behaviors for all items they may purchase.
    Importantly, such missingness is \textit{not at random}~\cite{li2023stabledr,schnabel2016recommendations}: auxiliary behaviors are observed for a skewed subset of items, driven by factors such as platform exposure policies and item popularity, rather than solely reflecting the absence of user interest.\footnote{The observability of interactions is influenced by multiple factors. Decomposing these causes is not our focus, and we collectively refer to this skew as observation bias~\cite{Bias_survey}.}
    GNNs propagating over such biased relations may overrepresent frequently exposed items with many interactions~\cite{APDA}.
    However, MBR must generalize beyond observed auxiliary relations and accurately recommend target behavior items without any auxiliary behaviors.
    This is illustrated by Figure~\ref{fig:figure1} (left).
    The training space defined by observed interactions covers only a limited subset of the inference space, while \textit{hidden target items}, i.e., target-positive items without auxiliary interactions, may appear outside the training space.
    Failure to capture such hidden preferences confines recommendations to items that users have already encountered, limiting their practical effectiveness.
    \end{itemize}
    
\begin{itemize}[leftmargin=*, labelsep=0.5em, topsep=0pt, itemsep=0pt]
    \item \textbf{C2: Unreliability of Auxiliary Signals.}
    Observed auxiliary behaviors do not necessarily imply purchase intent.
    While auxiliary behaviors may indicate potential interest, they can also be triggered by accidental clicks or casual browsing.
    As a result, auxiliary behaviors contain noisy signals that are misaligned with the target behavior.
    GNNs are structurally vulnerable to such noise because message passing can propagate and amplify incorrect information~\cite{10.1145/3488560.3498408,NoisyCL}.
    In MBR, auxiliary relation-driven propagation may amplify signals loosely related to the target behavior, thereby hindering target behavior prediction.
    Figure~\ref{fig:figure1} (left) illustrates this issue:
    not all items with auxiliary behaviors are target-positive items that eventually lead to purchase, as auxiliary interactions can also involve target-negative items.
\end{itemize}



\noindent
While prior methods partially address these challenges, they remain insufficient to fully resolve them, and to our knowledge, no existing method addresses both simultaneously.
Regarding \textbf{C1}, MEMBER~\cite{MEMBER} introduces self-supervised learning to reduce over-reliance on observed auxiliary interactions;
however, it still relies solely on the observed auxiliary signals, without explicitly uncovering hidden preferences.
Regarding \textbf{C2}, several studies~\cite{MuLe, HGIB} regulate message passing to suppress noise propagation;
however, the regulation is only weakly guided by the target behavior, which limits its effectiveness in filtering unreliable auxiliary signals.
Indeed, in Figure~\ref{fig:figure1} (right), we evaluate recent GNN-based methods relevant to each challenge (MEMBER~\cite{MEMBER} for C1, and MuLe~\cite{MuLe}, HGIB~\cite{HGIB} for C2) under settings that intensify each challenge by increasing the ratio of hidden target items and auxiliary noise, respectively.
Existing methods show substantial performance degradation, suggesting limited capability to handle challenges arising from auxiliary~behaviors.



As a solution, we propose \textbf{\proposed} (\underline{\textbf{B}}eyond \underline{\textbf{O}}bserved \underline{\textbf{A}}uxiliary \underline{\textbf{R}}elations), a framework that rewires given auxiliary relations into target-informative signals instead of using them as-is.
This rewiring requires different strategies depending on auxiliary observability:
when auxiliary observability is lacking, C1 calls for \textit{debiased densification} to surface hidden preferences, whereas with strong auxiliary observability, C2 calls for \textit{selective pruning} to filter unreliable signals.
Although both challenges could in principle be addressed within a single model, doing so may yield suboptimal signals under training data imbalance:
target interactions without auxiliary interactions form only a small fraction (e.g., $16.1\%$ in Tmall), causing the model to focus overly on auxiliary-observed cases and limiting hidden preference discovery (see Section~\ref{subsec:env_modeling} for detailed analysis).
This naturally motivates an environment-conditioned modular design:
\proposed employs two lightweight modules that focus on complementary rewiring strategies, i.e., recovering hidden preferences and pruning unreliable signals.
To coordinate these specialized modules, \proposed uses auxiliary observability as an environment condition and probabilistically adjusts their contributions to each target behavior prediction, addressing the two challenges in a complementary manner.
Our main contributions are:
\begin{itemize}[leftmargin=*] 
    \item We highlight two fundamental challenges of auxiliary behaviors that have not yet been fully addressed and, to our knowledge, present the first attempt to jointly address both.

    \item We propose \proposed, an environment-conditioned modular MBR framework that overcomes the limitations of observed auxiliary relations through conditional modular modeling.
    
    \item Extensive experiments show that \proposed delivers consistent gains, especially for target behavior without auxiliary observations, with efficiency comparable to state-of-the-art baselines.
\end{itemize}
\vspace{-\topsep}



\section{Related Works}
\noindent
\textbf{Multi-Behavior Recommendation (MBR).}
MBR leverages diverse auxiliary behaviors, such as clicks and add-to-cart, to alleviate the data sparsity of target behaviors like purchases. 
Early studies explored various techniques, including matrix factorization~\cite{MF-BPR, CMF}, deep neural networks~\cite{NMTR}, and attention-based approaches~\cite{DIPN, MATN, KHGT}.
Recently, GNN-based methods have emerged as the dominant approach~\cite{MBGCN,MBGMN,KHGT,CML,MBSSL}.
Existing methods adopt diverse graph encoding strategies, 
including integrating multiple behavior types into a unified graph~\cite{MBGCN,MGNN,MBGMN,GHCF,KHGT},
explicitly modeling natural behavioral sequences (e.g., click→cart→purchase)~\cite{CRGCN,MBCGCN,PKEF,MPC}, 
and encoding each behavior graph in parallel~\cite{MBHGCN,MuLe,HPMR,HGIB,MEMBER}.
Furthermore, advanced training techniques such as self-supervised learning~\cite{S-MBRec,MBSSL,COPF,HGIB,MEMBER} and multi-task learning~\cite{CRGCN,MBSSL,PKEF,HPMR,MPC,COPF} have been incorporated to enhance target behavior~prediction.


However, reliance on auxiliary behaviors introduces the two aforementioned challenges, and recent GNN-based methods have partially addressed them.
For C1, MEMBER~\cite{MEMBER} adopts a mixture-of-experts framework to distinguish purchases with auxiliary behaviors from those without.
It employs self-supervised learning to reduce the dominance of purchases preceded by auxiliary behaviors, encouraging the model to better capture purchase patterns that occur independently of auxiliary signals.
For C2, MuLe~\cite{MuLe} introduces attention-based message passing to adjust the contributions of auxiliary behaviors and mitigate uncertainty during propagation, 
while HGIB~\cite{HGIB} applies the information bottleneck principle to filter out information weakly related to the final prediction.

Despite these advancements, substantial room for improvement remains in addressing each challenge, and to our knowledge, no prior method has addressed both challenges simultaneously.
Specifically, no dedicated effort has been made to explicitly handle bias or uncover hidden target items, limiting the ability to address C1.
Meanwhile, existing methods control propagation primarily based on similarity induced by auxiliary behaviors, with limited guidance from target behaviors, which remains insufficient to resolve C2.

\smallsection{Graph Rewiring for Imperfect Graphs.}
In MBR, the two aforementioned challenges are fundamentally manifested as structural imperfections in auxiliary behavior graphs.
Graph rewiring (GR) is a core methodology in graph structure learning that reconstructs edge connections to address such imperfections~\cite{linkerhagner2025joint,gao2024graph,ding2022data, HDHGR}.
GR methods typically jointly optimize graph structure and node representations to recover missing edges and mitigate unreliable connections.
A common approach infers potential connections based on learned node embedding similarities, followed by pruning unreliable edges or adding new ones~\cite{EGLN,RGCF,SHaRe,bi2024make}.
These techniques have proven effective across domains, including single-behavior recommendation~\cite{EGLN, RGCF}, graph classification~\cite{bi2024make}, and social recommendation~\cite{SHaRe}.
Such improvements are attributed to pruning unreliable connections and strengthening homophilic relationships, enhancing graph quality.

However, applying existing GR strategies to MBR may be suboptimal, as auxiliary behavior graphs contain imperfections tied to auxiliary observability.
In particular, conventional similarity-based rewiring may amplify these imperfections:
densification may favor items with abundant auxiliary interactions, failing to surface hidden preferences without auxiliary evidence (C1), while pruning may mistakenly remove informative auxiliary connections as noise (C2).
This calls for an auxiliary observability-conditioned graph rewiring strategy tailored to the unique challenges of MBR.

\section{Preliminaries}
\subsection{Problem Formulation}
Let $\mathcal{U}$ and $\mathcal{I}$ denote the sets of users and items. 
We consider a set of behaviors $\mathcal{B}=\{b_1,\ldots,b_{|\mathcal{B}|}\}$ (e.g., click, cart, ... , buy), where $b_{\mathrm{tar}}=b_{|\mathcal{B}|}$ is the target behavior and $\mathcal{B}\setminus\{b_{\mathrm{tar}}\}$ are auxiliary behaviors.
For each behavior $b\in\mathcal{B}$, we define a bipartite graph $\mathcal{G}_b=(\mathcal{U}\cup\mathcal{I},\mathcal{E}_b)$, where $\mathcal{E}_b\subseteq\mathcal{U}\times\mathcal{I}$ denotes the observed interactions.
Let $\mathcal{E}_{\text{aux}}=\bigcup_{b\in \mathcal{B}\setminus\{b_{\mathrm{tar}}\}}\mathcal{E}_b$ denote the set of all observed auxiliary interactions.

For each pair $(u,i)$, we define a binary random variable $R_{ui}\in\{0,1\}$ indicating whether user $u$ engages in the target behavior on item~$i$.
Our goal is to learn a function $f(u,i)$ that estimates the probability that a user $u$ exhibits the target behavior with item $i$:
\begin{equation}
f(u,i)\ \approx\ \mathbb{E}\!\left[R_{ui}\mid u,i\right]\ =\ P(R_{ui}=1\mid u,i).
\end{equation}
We focus on addressing the aforementioned challenges, \textbf{C1}: missing auxiliary signals and \textbf{C2}: unreliable auxiliary signals, arising from the nature of auxiliary behaviors.

\begin{figure}[t]
    \centering
    \includegraphics[width=\columnwidth]{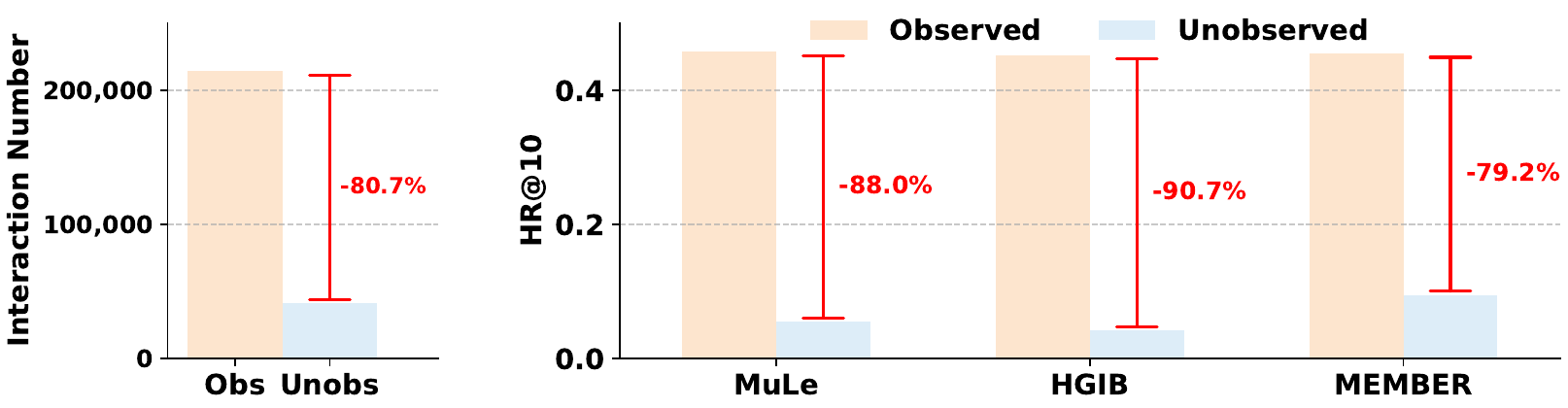}
    \caption{Training-set imbalance between target interactions with (Obs) and without (Unobs) auxiliary observations (left), and the resulting performance gap between the two groups (right). Results are reported on the Tmall dataset.}
    \label{fig:figure0}
    \vspace{-2em}
\end{figure}

\subsection{Environment-Conditioned Modeling}
\label{subsec:env_modeling}
As discussed above, the two challenges require distinct graph rewiring strategies: 
densifying the target behavior graph to recover hidden target items when auxiliary evidence is lacking, and pruning the auxiliary behavior graphs to suppress noise misaligned with target behavior when auxiliary evidence is abundant.

\noindent
\textbf{Empirical Evidence.}
Figure.~\ref{fig:figure0} (left) shows that among target interactions, those with auxiliary interactions substantially outnumber those without.
This imbalance can make GNN-based models overly focus on items with auxiliary observations.
Indeed, Figure.~\ref{fig:figure0} (right) shows that even state-of-the-art methods consistently suffer a significant performance drop on hidden target items.
This suggests that a single model may struggle to derive rewiring signals suited to both cases, particularly for hidden preference discovery.
This motivates a design that incorporates auxiliary observability, so that rewiring signals can be effectively derived for both auxiliary-observed and -unobserved cases.

\noindent
\noindent
\textbf{Environment-Conditioned Modeling.}
The above observations motivate a modular formulation in which target preference estimation is conditioned on auxiliary observability.
By the law of total expectation, the target preference can be decomposed as:
\begin{equation}
\label{eq:total_expectation}
\mathbb{E}[R_{ui}\mid u,i]
= \sum_{e\in\{0,1\}} P(E_{ui}{=}e\mid u,i)\ \mathbb{E}[R_{ui}\mid u,i,E_{ui}{=}e],
\end{equation}
where $E_{ui}$ denotes the auxiliary-observability environment of the pair $(u,i)$.
In the interaction log, the presence of auxiliary interactions provides an observable basis for this condition: $E_{ui}=0$ for $(u,i)\notin \mathcal{E}_{\text{aux}}$ and $E_{ui}=1$ for $(u,i)\in \mathcal{E}_{\text{aux}}$.
Note that $E_{ui}=0$ only means that no auxiliary behavior is observed in the current log; such pairs may become observable as user activity continues. 
Thus, the final prediction uses $P(E_{ui}{=}e\mid u,i)$ as a soft weight, rather than relying solely on the observed environment information.

Accordingly, we model the conditional expectations with two environment-conditioned sub-modules:
\begin{equation}
\begin{aligned}
f_0(u,i) &\approx \mathbb{E}[R_{ui}\mid u,i,E_{ui}{=}0],\\
f_1(u,i) &\approx \mathbb{E}[R_{ui}\mid u,i,E_{ui}{=}1].
\end{aligned}
\end{equation}
Here, $f_0$ focuses on rewiring under the auxiliary-unobserved condition, where hidden preferences should be recovered.
In contrast, $f_1$ focuses on rewiring under the auxiliary-observed condition, where unreliable auxiliary relations should be filtered.

The final prediction is obtained by softly aggregating the conditional predictions as $f(u,i)
= \sum_{e\in\{0,1\}} P(E_{ui}=e \mid u,i)\, f_e(u,i)$.
Here, $P(E_{ui}=e \mid u,i)$ controls the contribution of each sub-module.
This modular formulation enables coordination of complementary rewiring strategies across auxiliary-observability conditions, while behavior-specific importance is further taken into account within the modules as appropriate.
The concrete instantiation is introduced in the subsequent section.

\section{Proposed Method}
\begin{figure*}[t]
    \centering
    \includegraphics[width=\textwidth]{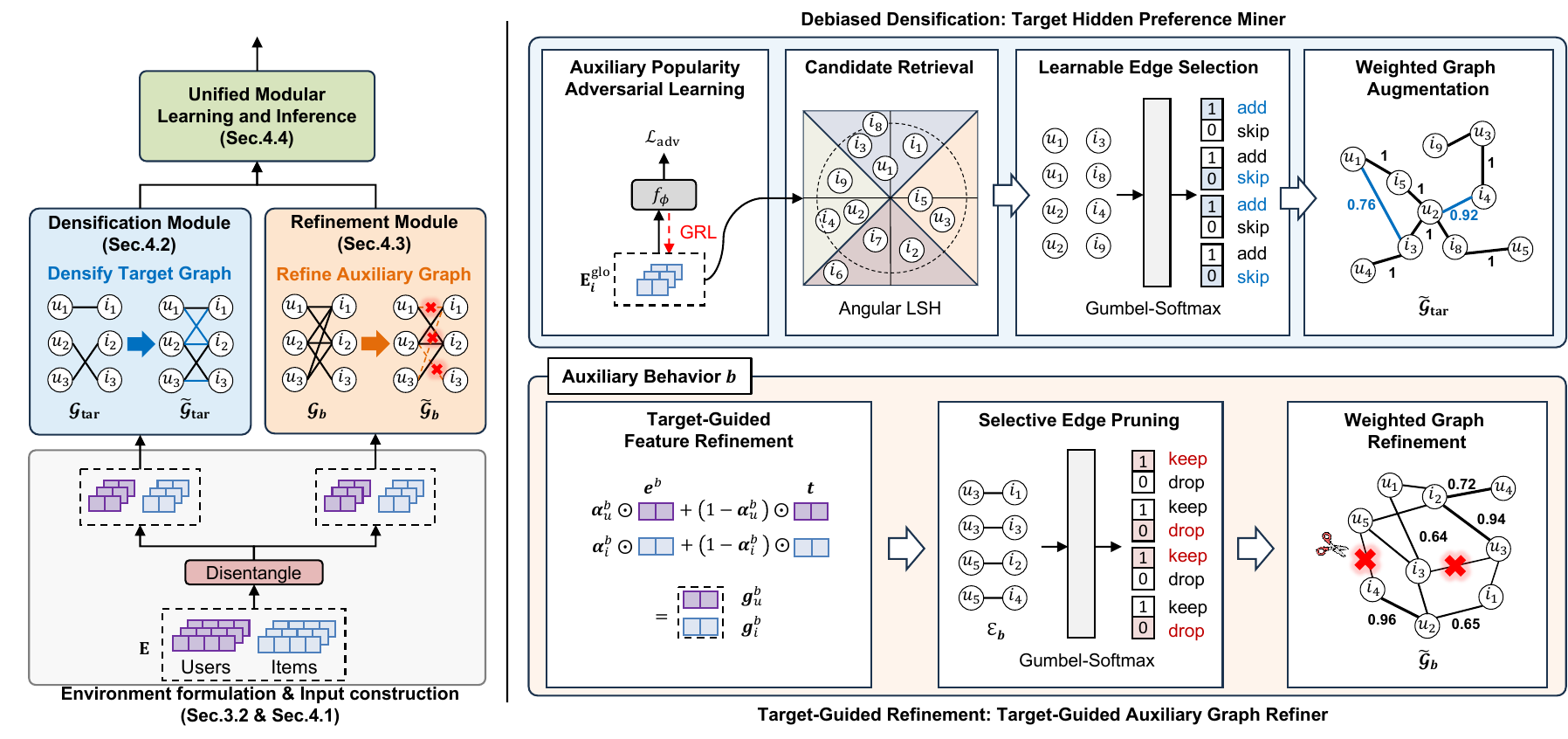}
    \caption{(Left) Overview of the \proposed framework. (Right) Construction of densification and refinement signals in each module, which are subsequently incorporated via self-supervised contrastive learning. Best viewed in color.}
    \label{fig:overview_method}
\end{figure*}

We present BOAR, an environment-conditioned MBR framework (Figure~\ref{fig:overview_method}).
\proposed first constructs disentangled input representations (Sec.~4.1), and then applies two modules, with distinct graph rewiring and learning strategies conditioned on the environment:
\begin{itemize}[leftmargin=*, labelsep=0.5em, topsep=3pt, itemsep=3pt]
    \item \textbf{(Sec.4.2) Debiased densification module $f_0$} identifies hidden target items by mining hidden preferences while debiasing auxiliary relations, providing densification signals.

    \item \textbf{(Sec.4.3) Target-guided refinement module $f_1$} prunes auxiliary relations misaligned with the target behavior, refining auxiliary graphs toward target-consistent relations.
\end{itemize}
Section~4.4 describes modular integration for learning and inference.


\subsection{\textbf{Input Representation Construction}}
\label{sec:unobserved_base}
We construct input user/item representations using both the global graph and the target-behavior graph via a LightGCN~\cite{LightGCN} encoder.
Let $\mathbf{E}\in\mathbb{R}^{(|\mathcal{U}|+|\mathcal{I}|)\times d}$ denote the initial embedding matrix that stacks user and item embeddings. 
We obtain:
\begin{equation}
\mathbf{E}^{\mathrm{glo}} = \mathrm{LGCN}\!\left(\mathbf{A}_{\mathrm{glo}}, \mathbf{E}\right), \qquad
\mathbf{E}^{\mathrm{tar}} = \mathrm{LGCN}\!\left(\mathbf{A}_{\mathrm{tar}}, \mathbf{E}\right),
\end{equation}
where $\mathbf{A}_{\mathrm{glo}}$ denotes the adjacency matrix of the global graph $\mathcal{G}_{\mathrm{glo}}$ constructed from all behaviors, i.e., $\mathcal{E}_{\mathrm{glo}}=\bigcup_{b\in\mathcal{B}} \mathcal{E}_b$,
and $\mathbf{A}_{\mathrm{tar}}$ denotes the adjacency matrix of the target behavior graph.
The global embeddings capture comprehensive multi-behavior signals and provide warm-start features for subsequent behavior-specific encoding~\cite{MuLe,HGIB,MBR_survey}. 
In each module, the global embedding $\mathbf{E}^{\mathrm{glo}}$ serves as a source of densification ($f_0$) or refinement ($f_1$) signals, while the target embedding $\mathbf{E}^{\mathrm{tar}}$ is used for final prediction.
Note that, to enable environment-conditioned modeling without increasing model capacity, we split the initial embedding into two halves and encode each half with both the global and target graph encoders.

\label{sec:unobserved}

\subsection{Debiased Densification Module}
The densification module $f_0$ mines hidden preferences to provide densification signals for the target behavior graph, and consists of three components:
(1) \textit{auxiliary popularity adversarial learning} to adversarially suppress popularity-driven observation bias from item representations,
(2) \textit{target hidden preference miner} to efficiently discover hidden preferences,
and (3) \textit{debiased densification learning} to safely inject the mined knowledge into the final prediction.





\subsubsection{\textbf{Auxiliary Popularity Adversarial Learning.}}
To reduce popularity-driven observation bias before hidden preference mining, we adapt the adversarial debiasing~\cite{Zhu0WW25,advdrop,fairrec}.
The key idea is to prevent item representations from overly encoding auxiliary behavior popularity, so that they can better reflect user preference.

This is achieved through adversarial learning with a gradient reversal layer (GRL)~\cite{ganin2016domain}.
Following~\cite{Zhu0WW25}, we construct a proxy signal for popularity-driven bias from the total number of auxiliary interactions each item receives:
$y_i = 1$ if the count exceeds the median interaction count across all items, and $y_i = 0$ otherwise.
This signal construction strategy provides practical and balanced adversarial supervision.\footnote{While more complex signals could be considered, prior work has shown that this median-based strategy can effectively suppress popularity-related information~\cite{Zhu0WW25}.}
We then define the adversarial supervision set as $\mathcal{S}_{\mathrm{adv}} = \{(i, y_i)\}$.

A discriminator $f_{\phi}$, implemented as a linear layer with a sigmoid function, is trained to predict the proxy signal $y_i$ of each item $i$.
Since popularity is an item-level property, the discriminator takes the item embedding as input and outputs $\hat{y}_i = \sigma(f_{\phi}(\mathrm{GRL}(\mathbf{e}^{\mathrm{glo}}_i)))$, where $\sigma$ is the sigmoid function.
The discriminator is trained adversarially via the GRL, such that the item representations are updated to hinder accurate popularity signal prediction:
\begin{equation}
\label{eq:adv}
\begin{gathered}
\mathcal{L}_{adv} = \mathbb{E}_{(i,y_i) \sim 
\mathcal{S}_{\mathrm{adv}}} \left[ -y_i \log \hat{y}_i - (1-y_i) \log (1-\hat{y}_i) \right]
\end{gathered}
\end{equation}
This adversarial objective reduces the influence of popularity information in item representations, providing a basis for hidden preference mining.




\subsubsection{\textbf{Target Hidden Preference Miner.}}
\label{subsubsec:thpm}
Based on the popularity-suppressed representations, we uncover potential target relations.
Existing graph rewiring methods~\cite{SHaRe,HDHGR} typically rely on exhaustive pairwise similarity computation to identify highly similar node pairs, incurring prohibitive costs.
To address this, we propose an efficient three-stage mining strategy: 
(i) locality-sensitive hashing (LSH)-based candidate retrieval, 
(ii) learnable edge selection, and (iii) weighted graph augmentation.



\smallsection{Angular LSH-based Candidate Retrieval.}
To efficiently retrieve candidate items, we adopt angular LSH ~\cite{lsh,reformer}, which enables fast approximate nearest-neighbor search based on angular similarity.
Unlike traditional LSH, angular LSH constructs hash buckets using a single random projection followed by an argmax operation, which can be efficiently computed on GPUs via matrix multiplication.
Moreover, its collision probability is monotonically related to angular similarity rather than L2 distance, making it better aligned with the inner product–based prediction.

We first draw a random projection matrix $\mathbf{R} \in \mathbb{R}^{d \times d'}$ with i.i.d. entries $R_{pq} \sim \mathcal{N}(0,1)$, where $d'$ is the number of hash buckets and each column defines a random angular direction in the embedding space.
For each node $x$, we normalize its representation and project it onto these directions, obtaining projection scores $\mathbf{s}_x$ that reflect its angular alignment.\footnote{Concretely, we compute $\mathbf{s}_x = [ (\mathbf{e}^{\mathrm{glo}}_x / |\mathbf{e}^{\mathrm{glo}}_x|)^\top \mathbf{R} \, \Vert \, (\mathbf{e}^{\mathrm{glo}}_x / |\mathbf{e}^{\mathrm{glo}}_x|)^\top (-\mathbf{R}) ] \in \mathbb{R}^{2d'}$.
The concatenation with $-\mathbf{R}$ allows us to capture both positive and negative angular directions using a single argmax operation, resulting in $2d'$ effective buckets~\cite{reformer}.}
Each node is then assigned to a hash bucket corresponding to the dominant projection direction as: $h(x) = \arg\max_k (\mathbf{s}_x)_k$.
This procedure ensures that nodes with similar angular directions are likely to be grouped into the same bucket, without explicitly computing pairwise similarities.

For each user $u$, we find unobserved items from the same bucket, i.e., $\{ i \mid h(i) = h(u)\}$, and retrieve top-$n$ items by cosine similarity to form the candidate set $\mathcal{E}_{\mathrm{cand}}(u)$.
The total candidates across all users are aggregated as:
$\mathcal{E}_{\mathrm{cand}} = \bigcup_{u \in \mathcal{U}} \{ (u,i) \mid i \in \mathcal{E}_{\mathrm{cand}}(u) \}$.

\smallsection{Learnable Edge Selection.}
Not all candidates are equally reliable for target behavior prediction.
To selectively retain edges truly aligned with this objective, we adopt a learnable edge selection strategy jointly optimized with the prediction task.

For each candidate pair $(u,i) \in \mathcal{E}_{\mathrm{cand}}$, we predict a binary add-or-skip decision using $f_{\mathrm{add}}$, implemented as a linear layer.
Specifically, we compute logits $\mathbf{z}^{\mathrm{add}}_{ui} = f_{\mathrm{add}}([\mathbf{e}^{\mathrm{glo}}_u \parallel \mathbf{e}^{\mathrm{glo}}_i]) \in \mathbb{R}^2$, where each output dimension corresponds to add and skip decisions, respectively.
We apply the Gumbel-Softmax~\cite{gumbel} to obtain differentiable selection~mask:\footnote{
$\mathrm{GumbelSoftmax}(\mathbf{z}; \tau)_k = \frac{\exp((z_k + g_k)/\tau)}{\sum_j \exp((z_j + g_j)/\tau)}$,
where $g_k = -\log(-\log(u_k))$ with $u_k \sim \mathrm{Uniform}(0,1)$.
As $\tau \rightarrow 0$, the output approaches a one-hot discrete sample.
}
\begin{equation}
\begin{split}
m^{\mathrm{add}}_{ui} &= \mathrm{GumbelSoftmax}(\mathbf{z}^{\mathrm{add}}_{ui}; \tau)[0] \in \{0,1\}, \\
\end{split}
\end{equation}
The selection mask is learned via backpropagation using the Gumbel-Softmax relaxation, favoring edges beneficial for target behavior prediction while suppressing misaligned ones.
We guide readers unfamiliar with the Gumbel-Softmax~to~\cite{gumbel}.


\smallsection{Weighted Graph Augmentation.}
We construct the densified target graph using the selected edges.
To reflect their varying importance, each edge is assigned a weight $w_{ui}$ that combines angular similarity and magnitude consistency~\cite{chen2024dualdomaincollaborativedenoisingsocial}:
\begin{equation}
w_{ui} = \frac{1}{2}\Big(1 + \cos(\mathbf{e}^{\mathrm{glo}}_u, \mathbf{e}^{\mathrm{glo}}_i)\Big) \cdot \exp\!\left(-\frac{1}{2\sigma^2} \|\mathbf{e}^{\mathrm{glo}}_u - \mathbf{e}^{\mathrm{glo}}_i\|_2^2\right).
\label{eq:edge_weight} 
\end{equation}
Original target edges are retained with unit weight, while selected candidates are added with weight $w_{ui}$, where $\sigma$ controls the sensitivity to the Euclidean distance between embeddings, and is set to 20 following~\cite{chen2024dualdomaincollaborativedenoisingsocial}.
We define the weighted adjacency matrix $\tilde{\mathbf{A}}_{\mathrm{tar}}$ as:
\begin{equation}
\tilde{\mathbf{A}}_{\mathrm{tar}}(u,i)=
\begin{cases}
1, & (u,i)\in\mathcal{E}_{\mathrm{tar}},\\
w_{ui}, & (u,i)\in\mathcal{E}_{\mathrm{cand}} \ \,\, \text{and}\,\, \ m^{\mathrm{add}}_{ui}=1,\\
0, & \text{otherwise}.
\end{cases}
\end{equation}
The matrix $\tilde{\mathbf{A}}_{\mathrm{tar}}$ is symmetric, as $w_{ui} = w_{iu}$ by construction.


\subsubsection{\textbf{Debiased Densification Learning}}
While the augmented graph reveals hidden preferences, it inevitably contains noise from the mining process, making direct use for final prediction suboptimal.
We propose a strategy to robustly exploit densification~signals.

\smallsection{Preference Densification Learning.}
We adopt self-supervised contrastive learning with an InfoNCE loss~\cite{oord2018representation} to encourage consistency between embeddings from the original and augmented graphs.
This alignment emphasizes signals consistently supported by both views, enabling a more stable learning than directly treating augmented edges as ground-truth relations.
Specifically, we align the original target embeddings $\mathbf{E}^{\mathrm{tar}}$ (input of this module) with the augmented embeddings $\tilde{\mathbf{E}}^{\mathrm{tar}} = \mathrm{LGCN}(\tilde{\mathbf{A}}_{\mathrm{tar}}, \mathbf{E}^{\mathrm{glo}})$.
The user-side loss for densification is defined as:
\begin{equation}
\mathcal{L}^{user}_{dense}
= - \sum_{u\in\mathcal{U}}
\log
\frac{\exp\big(\mathrm{sim}(\mathbf{e}^{\mathrm{tar}}_{u},\tilde{\mathbf{e}}^{\mathrm{tar}}_{u})/\tau\big)}
{\sum_{u'\in\mathcal{U}} \exp\big(\mathrm{sim}(\mathbf{e}^{\mathrm{tar}}_{u},\tilde{\mathbf{e}}^{\mathrm{tar}}_{u'})/\tau\big)}.
\end{equation}
where $\mathrm{sim}(\cdot,\cdot)$ denotes a similarity function and $\tau$ is the temperature hyperparameter.
The item-side loss $\mathcal{L}^{item}_{dense}$ is defined analogously, and the overall objective is $\mathcal{L}_{dense}=\mathcal{L}^{user}_{dense}+\mathcal{L}^{item}_{dense}$.

\smallsection{Debiased BPR Loss.}
While auxiliary popularity adversarial learning mitigates popularity-driven bias at the representation level, it does not fully eliminate observation bias in preference learning.
For a more explicit debiasing, we adopt the self-normalized inverse propensity score (SNIPS)~\cite{schnabel2016recommendations} to reweight positive interactions in the ranking loss.
SNIPS reduces the variance of the IPS estimator by normalizing importance weights, thereby stabilizing optimization even under imperfect propensity estimates.



For each $(u,i)$, we first estimate the auxiliary-observation propensity $\hat{p}_{ui} = \hat{P}(E_{ui}=1 \mid u,i)$, which measures how likely the pair is to be observed in auxiliary behaviors.\footnote{Following~\cite{schnabel2016recommendations, li2023stabledr, Zhu0WW25}, 
we adopt a simple propensity estimator trained via binary cross-entropy on auxiliary-observed (positive) and -unobserved (negative) pairs: 
$\mathcal{L}_{\mathrm{BCE}} = \frac{1}{|\mathcal{D}|}
\sum_{(u,i)\in\mathcal{D}} \mathrm{BCE}(\hat{p}_{ui},\, E_{ui})$.
Details are provided in Appendix~\ref{app:propensity}.}
The debiasing weight is defined by the normalized inverse propensity, $\omega_{ui} \propto 1/\hat{p}_{ui}$, which reduces estimator variance.
Intuitively, $\omega_{ui}$ assigns smaller weights to pairs with high auxiliary observation propensity and larger weights to those with low propensity, correcting the skew induced by observation bias.
Let $s_{ui} = {\mathbf{e}}^{\mathrm{tar}}_{u}{}^{\top}{\mathbf{e}}^{\mathrm{tar}}_{i}$ denote the prediction score for target behavior. 
The debiased BPR loss is:
\begin{equation}
\label{eq:dbpr}
\mathcal{L}_{rank}
= - \sum_{(u,i)\in\mathcal{E}_{\mathrm{tar}}}
  \sum_{j:\,(u,j)\notin\mathcal{E}_{\mathrm{tar}}}
  \omega_{ui} \cdot \log \sigma(s_{ui} - s_{uj})
\end{equation}

\smallsection{Overall Learning Objective.}
The total loss of $f_0$ is as follows:
\begin{equation}
\mathcal{L}_{f_0}
= \mathcal{L}_{rank}
+ \lambda_{adv} \mathcal{L}_{adv}
+ \lambda_{dense} \mathcal{L}_{dense},
\end{equation}
where $\lambda_{adv}$ and $\lambda_{dense}$ are loss-balancing hyperparameters.

Theoretical grounding for $f_0$ is provided in Appendix~\ref{app:ips-theory}: \textit{(i)} Theorem~\ref{thm:ips} establishes unbiasedness of the IPS-weighted loss;
\textit{(ii)} Lemma~\ref{lem:hdiv} shows that the GRL objective suppresses popularity-driven bias from item representations; and \textit{(iii)} Corollary~\ref{cor:grl-ips} shows that both are jointly necessary for hidden preference recovery.

\label{sec:observed}
\subsection{Target-Guided Refinement Module}
The refinement module $f_1$ leverages target behavior signals to prune the auxiliary behavior graphs, and consists of two components:
(1) a \textit{target-guided auxiliary graph refiner}, which selectively prunes and reweights auxiliary relations, and
(2) \textit{target-guided preference learning} that provides target-aligned signals for auxiliary refinement and optimizes target behavior prediction.

\subsubsection{\textbf{Target-guided Auxiliary Graph Refiner}}

We perform target-guided refinement on each auxiliary behavior graph $\mathcal{G}_{b}$ using a behavior-specific LGCN encoder, with input auxiliary embeddings $\mathbf{e}^{b}=\mathbf{e}^{\mathrm{glo}}$.
The target embeddings $\mathbf{E}^{\mathrm{tar}}$ serve as fixed anchors, and for notational simplicity, we denote $\mathbf{t}_{x}=\mathrm{detach}(\mathbf{e}^{\mathrm{tar}}_{x})$.

\smallsection{Target-Guided Feature Refinement.}
First, to handle feature-level misalignment, we introduce a feature-wise gate that adaptively balances auxiliary representations with target anchors.
For each node $x$, the gate $\boldsymbol{\alpha}^{b}\in(0,1)^d$ determines, at each feature dimension, how much to rely on the auxiliary signals versus the target anchor.
The gate is computed by conditioning on both representations:
\begin{equation}
\begin{aligned}
    \boldsymbol{\alpha}^{b}_{x}
    &= \sigma \Big( \mathbf{W}^{b}_{1}\big[\mathbf{e}^{b}_{x}\,\|\,\mathbf{t}_{x}\big] + \mathbf{b}^{b}_{1} \Big)
    \in (0,1)^{d}\\
    \mathbf{g}^{b}_{x} &= \boldsymbol{\alpha}^{b}_{x}\odot \mathbf{e}^{b}_{x}
    + \big(\mathbf{1}-\boldsymbol{\alpha}^{b}_{x}\big)\odot \mathbf{t}_{x},
\end{aligned}
\end{equation}
where $x\in\mathcal{U}\cup\mathcal{I}$.
This adaptive interpolation selectively filters auxiliary features inconsistent with the target preference space.


\smallsection{Selective Edge Pruning.}
Beyond feature-level refinement, we further refine the auxiliary graph structure by selectively pruning uninformative edges.
The goal is to prevent auxiliary edges inconsistent with the target preference from propagating noise.
For each edge $(u,i)\in\mathcal{E}_{b}$, we predict a binary keep-or-drop decision using $f_{\mathrm{drop}}$, implemented as a linear layer.
Specifically, we compute logits $\mathbf{z}^{b}_{ui} = f_{\mathrm{drop}}([\mathbf{g}^{b}_{u}\ \Vert\ \mathbf{g}^{b}_{i}]) \in \mathbb{R}^2$, where each output dimension corresponds to keep and drop decisions, respectively.
We apply Gumbel-Softmax to obtain differentiable pruning:
\begin{equation}
m^{b}_{ui} = \mathrm{GumbelSoftmax}(\mathbf{z}^{b}_{ui}; \tau)[0] \in \{0,1\},
\end{equation}
Optimized by the target-guided learning introduced in the subsequent subsection, edges misaligned with the target behavior are suppressed.
An edge is retained if $m^{b}_{ui}=1$; otherwise it is pruned.


\smallsection{Weighted Graph Refinement.}
We construct a refined adjacency matrix $\tilde{\mathbf{A}}_{b}$ 
by selecting a subset of edges with the weighting scheme in Eq.~\eqref{eq:edge_weight}, substituting refined representations $\mathbf{g}^{b}_{u}$, $\mathbf{g}^{b}_{i}$ for $\mathbf{e}^{\mathrm{glo}}_{u}$, $\mathbf{e}^{\mathrm{glo}}_{i}$:
\begin{equation}
\tilde{\mathbf{A}}_{b}(u,i)=
\begin{cases}
w^{b}_{ui}, & (u,i)\in\mathcal{E}_{b}',\\
0, & \text{otherwise},
\end{cases}
\end{equation}
where $\mathcal{E}_{b}'=\{(u,i)\in\mathcal{E}_{b}\mid m^{b}_{ui}=1\}$ 
denotes the set of retained edges after pruning.
The resulting matrix $\tilde{\mathbf{A}}_{b}$ is symmetric, as $w^{b}_{ui} = w^{b}_{iu}$ by construction.
Given the refined adjacency matrix $\tilde{\mathbf{A}}_{b}$, we obtain propagated representations as:
$\tilde{\mathbf{E}}^{b} = \mathrm{LGCN}(\tilde{\mathbf{A}}_{b}, \mathbf{E}^{\mathrm{glo}})$.
The same refinement procedure is applied to each auxiliary behavior $b$.

\subsubsection{\textbf{Target-Guided Preference Learning}}
We introduce a target-guided preference learning that leverages contrastive alignment to provide supervisory signals for refining auxiliary graphs and strengthening target behavior prediction.

\smallsection{Auxiliary Refinement Learning.}
Specifically, we treat the target embedding $\mathbf{t}_{x}$ as a fixed anchor to prevent auxiliary noise from contaminating the target space, and align each auxiliary representation with it via contrastive learning.
The user-side refinement loss for auxiliary behavior $b$ is:
\begin{equation}
\mathcal{L}^{user,b}_{refine}
= - \sum_{u\in\mathcal{U}}
\log
\frac{\exp\big(\mathrm{sim}(\mathbf{t}_{u},\mathbf{e}^{b}_{u})/\tau\big)}
{\sum_{u'\in\mathcal{U}} \exp\big(\mathrm{sim}(\mathbf{t}_{u},\mathbf{e}^{b}_{u'})/\tau\big)}.
\end{equation}
The item-side loss is defined analogously, and we sum both losses over all auxiliary behaviors: $\mathcal{L}_{refine}=\sum_{b\in\mathcal{B}\setminus\{b_\mathrm{tar}\}} \mathcal{L}^{user,b}_{refine}+\mathcal{L}^{item,b}_{refine}$.

\smallsection{Target Preference Learning.}
We aggregate behavior-specific embeddings according to their relevance to the target behavior.
Let $a(\cdot): \mathbb{R}^{2d} \rightarrow \mathbb{R}$ denote an attention function, implemented as a linear layer, over the concatenation of two embeddings.
\begin{equation}
\begin{aligned}
\alpha_u^{b} =
\frac{\exp\!\big(a(\mathbf{e}^{\mathrm{tar}}_{u},\mathbf{e}^{b}_{u})\big)}
{\sum_{b'\in\mathcal{B}}\exp\!\big(a(\mathbf{e}^{\mathrm{tar}}_{u},\mathbf{e}^{b'}_{u})\big)}, \quad \mathbf{e}^{\mathrm{agg}}_{u} =
\sum_{b\in\mathcal{B}} \alpha_u^{b}\,\mathbf{e}^{b}_{u},
\end{aligned}
\end{equation}
Based on the aggregated representations, we compute the preference score as
$s_{ui} = \mathbf{e}_{u}^{\text{agg}\top}\mathbf{e}_{i}^{\text{agg}}$,
and optimize the BPR loss on the target behavior: $\mathcal{L}_{rank} = - \sum_{(u,i)\in\mathcal{E}_{\mathrm{tar}}}  \sum_{j:(u,j)\notin\mathcal{E}_{\mathrm{tar}}} \log\sigma(s_{ui} - s_{uj})$.

\smallsection{Overall Learning Objective.}
The total loss of $f_1$ is as follows:
\begin{equation}
\mathcal{L}_{f_1} = \mathcal{L}_{rank} + \lambda_{refine} \mathcal{L}_{refine},
\end{equation}
where $\lambda_{refine}$ is a loss-balancing hyperparameter.



\subsection{Unified Modular Learning and Inference}
According to our environment-conditioned design (Sec.~\ref{subsec:env_modeling}), the final objective is expressed as a conditional expectation over module-wise objectives.
Under this formulation, the auxiliary-observation propensity $\hat{p}_{ui}=\hat{P}(E_{ui}=1\mid u,i)$ and its complement $1-\hat{p}_{ui}$ naturally serve as the assignment probabilities for the~two~modules:\footnote{With a slight abuse of notation, $\mathcal{L}_{f_e}(u,i)$ denotes the loss induced by $(u,i)$ under~$f_e$.}
\begin{equation}
\mathcal{L}_{BOAR}=\mathbb{E}_{(u,i)\in\mathcal{E}_{\mathrm{tar}}}
\!\left[(1-\hat{p}_{ui})\,\mathcal{L}_{f_0}(u,i) + \hat{p}_{ui}\,\mathcal{L}_{f_1}(u,i)\right]
\end{equation}
Thus, each target interaction is softly assigned to the two modules according to its auxiliary-observation likelihood.

An alternative design is to use a gating network, jointly optimized with the model, to assign module weights from user--item embeddings. 
However, we found that propensity-based assignment is more effective, as it is grounded in the environment signal of auxiliary observation rather than learned as an unconstrained gate.
We highlight that propensity estimation is well established, and its stability and robustness to noise have been extensively studied~\cite{rosenbaum1983central,schnabel2016recommendations,li2023stabledr,Zhu0WW25}. 

\noindent\textbf{Inference.}
Consistent with the training objective, the final prediction is obtained as:
\begin{equation}
\label{eq:inference}
f(u,i) = (1-\hat{p}_{ui})\, f_0(u,i) + \hat{p}_{ui}\, f_1(u,i).
\end{equation}
This soft assignment allows each instance to leverage complementary signals from both modules.
Alternatively, replacing $\hat{p}_{ui}$ with the hard indicator $E_{ui}$ yields a deterministic 0-1 assignment, where each instance is routed exclusively to one module and cannot benefit from the other.
We validate this design choice in Sec.~\ref{subsub:ablation}.




\section{Experiments}
\subsection{Experimental Settings}
\begin{table}[t]
  \centering
  \caption{Dataset statistics. Hidden (\%) denotes the proportion of hidden target items without auxiliary behaviors.}
  \label{tab:data_statistics}
  \begingroup
  \setlength{\tabcolsep}{4pt}
  \renewcommand{\arraystretch}{1.1}
  \small
  \begin{threeparttable}
  \resizebox{\linewidth}{!}{%
    \begin{tabular}{lccccccc}
      \toprule
      \textbf{Dataset} & \textbf{\#Users} & \textbf{\#Items} & \textbf{\#Clicks} & \textbf{\#Collects} & \textbf{\#Carts} & \textbf{\#Buys} & \textbf{Hidden (\%)\tnote{*}} \\
      \midrule
      \textbf{Tmall}  & 41,738 & 11,953 & 1,813,498 & 221,514 & 1,996   & 255,586 & 16.38 \\
      \textbf{Taobao} & 48,749 & 39,493 & 1,548,162 & --      & 193,747 & 211,022 & 35.78 \\
      \textbf{JData}  & 93,334 & 24,624 & 1,681,430 & 45,613  & 49,891  & 321,883 & 18.55 \\
      \bottomrule
    \end{tabular}%
  }
  \begin{tablenotes}
    \footnotesize
    \item[*] Computed over test interactions, distinct from the training-set ratio in Fig~\ref{fig:figure0}.
  \end{tablenotes}
  \end{threeparttable}
  \endgroup
\end{table}
\noindent\textbf{Datasets.}
We use three widely used benchmark MBR datasets: (1) \textbf{Tmall}, from Alibaba platform and involving four behaviors (click, collect, cart, and buy); 
(2) \textbf{Taobao}, from Taobao and consisting of three behaviors (click, cart, and buy); 
and (3) \textbf{JData}, released by JD.com and containing four behaviors (click, collect, cart, and buy).
For all datasets, we treat \textit{buy} as the target behavior. 
Dataset statistics are summarized in Table~\ref{tab:data_statistics}.

\begin{table*}[t]
  \centering
  \caption{Performance comparison. 
  $^*$ indicates statistical significance for $p<0.05$ under t-test compared to the best baseline.}
  \label{tab:overall_performance}

  \begingroup
  \setlength{\tabcolsep}{2.5pt} 
  \renewcommand{\arraystretch}{1}
    
  \resizebox{\textwidth}{!}{%
    \begin{tabular}{lllcccccccccccccc}
      \toprule
      \multirow{2}{*}{\textbf{Dataset}} &
      \multirow{2}{*}{\textbf{Types}} &
      \multirow{2}{*}{\textbf{Metric}} &
      \multicolumn{2}{c}{\textbf{Single-Behavior}} &
      \multicolumn{12}{c}{\textbf{Multi-Behavior}} \\
      \cmidrule(lr){4-5} \cmidrule(lr){6-17}
      & & & \textbf{MF-BPR} & \textbf{LGCN} &
      \textbf{LGCN-G} & \textbf{MB-GMN} & \textbf{CIGF} & \textbf{CRCGN} &
      \textbf{BCIPM} & \textbf{MB-HGCN} & \textbf{COPF}$\,\,$ & \textbf{MuLe}$\,\,$ &
      \textbf{HGIB} & \textbf{MEMBER} & \textbf{SHaRe}& \textbf{BOAR} \\
      \midrule

      \multirow{6}{*}{\textbf{Tmall}} & \multirow{2}{*}{\textbf{General}} &
      \textbf{HR@10}   & 0.0427 & 0.0391 & 0.1339 & 0.0452 & 0.0621 & 0.0813 & 0.1407 & 0.1462 & 0.1640 & 0.2088 & 0.2415 & \underline{0.3507} & 0.2158 & \textbf{0.3739$^*$} \\
      & & \textbf{NDCG@10} & 0.0210 & 0.0201 & 0.0681 & 0.0228 & 0.0320 & 0.0428 & 0.0765 & 0.0778 & 0.0883 & 0.1158 & 0.1274 & \underline{0.1730} & 0.1133 & \textbf{0.1821$^*$} \\ \cmidrule(lr){2-17}
      & \multirow{2}{*}{\textbf{Unobserved}} &
      \textbf{HR@10}   & 0.0662 & 0.0692 & 0.0835 & 0.0298 & 0.0474 & 0.0327 & 0.0548 & 0.0422 & 0.0426 & 0.0549 & 0.0420 & \underline{0.0948} & 0.0727 & \textbf{0.1367$^*$} \\
      & & \textbf{NDCG@10} & 0.0335 & 0.0382 & 0.0425 & 0.0154 & 0.0250 & 0.0174 & 0.0282 & 0.0223 & 0.0224 & 0.0289 & 0.0213 & \underline{0.0594} & 0.0398 & \textbf{0.0812$^*$} \\ \cmidrule(lr){2-17}
      & \multirow{2}{*}{\textbf{Observed}} &
      \textbf{HR@10}   & 0.3996 & 0.3773 & 0.4256 & 0.3740 & 0.3617 & 0.3841 & 0.4075 & 0.3762 & 0.4048 & \underline{0.4557} & 0.4527 & 0.4554 & 0.4438 & \textbf{0.4587$^*$} \\
      & & \textbf{NDCG@10} & 0.1891 & 0.1765 & 0.2044 & 0.1834 & 0.1705 & 0.1832 & 0.1953 & 0.1796 & 0.1935 & 0.2224 & 0.2202 & \underline{0.2226} & 0.2149 & \textbf{0.2232$^*$} \\
      \midrule

      \multirow{6}{*}{\textbf{Taobao}} & \multirow{2}{*}{\textbf{General}} &
      \textbf{HR@10}   & 0.0239 & 0.0178 & 0.1049 & 0.0498 & 0.0649 & 0.1174 & 0.1270 & 0.1639 & 0.1543 & 0.2121 & 0.2480 & \underline{0.3183} & 0.1842 & \textbf{0.3432$^*$} \\
      & & \textbf{NDCG@10} & 0.0130 & 0.0095 & 0.0591 & 0.0198 & 0.0327 & 0.0655 & 0.0711 & 0.0950 & 0.0895 & 0.1346 & 0.1572 & \underline{0.1640} & 0.1113 & \textbf{0.1886$^*$} \\ \cmidrule(lr){2-17}
      & \multirow{2}{*}{\textbf{Unobserved}} &
      \textbf{HR@10}   & 0.0167 & 0.0124 & 0.0212 & 0.0182 & 0.0206 & 0.0177 & 0.0148 & 0.0153 & 0.0182 & 0.0248 & \underline{0.0252} & 0.0221 & 0.0246 & \textbf{0.0352$^*$} \\
      & & \textbf{NDCG@10} & 0.0093 & 0.0068 & 0.0105 & 0.0094 & 0.0102 & 0.0092 & 0.0077 & 0.0088 & 0.0075 & 0.0137 & \underline{0.0147} & 0.0118 & 0.0140 & \textbf{0.0196$^*$} \\ \cmidrule(lr){2-17}
      & \multirow{2}{*}{\textbf{Observed}} &
      \textbf{HR@10}   & 0.4587 & 0.4619 & 0.4938 & 0.4736 & 0.4849 & 0.4941 & 0.5024 & 0.4840 & 0.5343 & 0.5661 & \textbf{0.5797} & 0.5279 & 0.5435 & \underline{0.5697} \\
      & & \textbf{NDCG@10} & 0.2252 & 0.2290 & 0.2458 & 0.2318 & 0.2327 & 0.2530 & 0.2503 & 0.2446 & 0.2695 & 0.3050 & \textbf{0.3238} & 0.2738 & 0.2883 & \underline{0.3061} \\
      \midrule

      \multirow{6}{*}{\textbf{Jdata}} & \multirow{2}{*}{\textbf{General}} &
      \textbf{HR@10}   & 0.3674 & 0.2779 & 0.4163 & 0.2815 & 0.3650 & 0.4872 & 0.5252 & 0.5234 & 0.4497 & 0.5837 & 0.6502 & \underline{0.6589} & 0.6378 & \textbf{0.6971$^*$} \\
      & & \textbf{NDCG@10} & 0.2234 & 0.1704 & 0.2411 & 0.1652 & 0.2272 & 0.2894& 0.3158 & 0.3402 & 0.2723 & 0.4209 & \underline{0.4667} & 0.4332 & 0.4584 & \textbf{0.4937$^*$} \\ \cmidrule(lr){2-17}
      & \multirow{2}{*}{\textbf{Unobserved}} &
      \textbf{HR@10}   & 0.3610 & 0.2679 & 0.3671 & 0.2295 & 0.2971 & 0.3671 & 0.3740 & 0.4164 & 0.2610 & \underline{0.4749} & 0.4566 & 0.3947 & 0.4641 & \textbf{0.5438$^*$} \\
      & & \textbf{NDCG@10} & 0.2002 & 0.1432 & 0.2103 & 0.1415 & 0.1854 & 0.2029 & 0.2272 & 0.2588 & 0.1743 & 0.2975 & 0.2769 & 0.2403 & \underline{0.3018} & \textbf{0.3665$^*$} \\ \cmidrule(lr){2-17}
      & \multirow{2}{*}{\textbf{Observed}} &
      \textbf{HR@10}   & 0.7324 & 0.7189 & 0.7100 & 0.6937 & 0.7150 & 0.7452 & 0.7363 & 0.7234 & 0.7279 & 0.7693 & \textbf{0.7826} & 0.7383 & 0.7663 & \underline{0.7757} \\
      & & \textbf{NDCG@10} & 0.3751 & 0.4461 & 0.4372 & 0.4396 & 0.4427 & 0.4586 & 0.4458 & 0.4419 & 0.4648 & 0.4749 & \textbf{0.5340} & 0.4920 & 0.5235 & \underline{0.5273} \\
      \bottomrule
    \end{tabular}%
  }
  \endgroup
\end{table*}

\smallsection{Baselines.}
We compare \proposed with \textbf{13} baselines, covering both single- and multi-behavior recommendation methods.
Specifically, we include two single-behavior models—MF-BPR~\cite{BPR} 
and LightGCN (LGCN)~\cite{LightGCN}—and eleven multi-behavior methods: LGCN-G~\cite{LightGCN}, MB-GMN~\cite{MBGMN}, CIGF~\cite{CIGF}, CRGCN~\cite{CRGCN}, BCIPM~\cite{BCIPM}, MB-HGCN~\cite{MBHGCN},  COPF~\cite{COPF}, MuLe~\cite{MuLe}, HGIB~\cite{HGIB}, MEMBER~\cite{MEMBER}, and SHaRe~\cite{SHaRe}.
Among these, HGIB and SHaRe are graph rewiring-based methods; HGIB focuses on pruning-based rewiring to filter noisy auxiliary interactions, while SHaRe applies similarity-based rewiring to MBR, performing pruning on auxiliary behavior graphs and densification on the target behavior graph, serving as a direct rewiring baseline.


\smallsection{Evaluation Protocol.}
We closely follow the evaluation protocols of prior work~\cite{MBGMN,CIGF,BCIPM,COPF,MuLe,HGIB}.
We follow the widely adopted leave-one-out protocol, which holds out the most recent interaction of each user, together with all uninteracted items, as the test set.
The second most recent interaction is reserved as the validation set. 
We evaluate performance using two ranking metrics: Hit Ratio (HR@$k$) and Normalized Discounted Cumulative Gain (NDCG@$k$).
Following prior studies~\cite{BCIPM,MuLe,HGIB,COPF}, we set $k=10$ in all experiments.
We report the average performance of five independent runs, each of which uses different random seeds.


\smallsection{Performance Breakdown by Auxiliary Observability.}
For a more comprehensive evaluation, we report performance under three settings based on the observability of \textit{training auxiliary behaviors}: \textbf{general}, \textbf{observed}, and \textbf{unobserved}.
The general setting averages performance over \textit{all} test instances.
The observed setting reports the average over test instances with \textit{at least one} auxiliary interaction (i.e., observed target items), whereas the unobserved setting does so over test instances \textit{without} auxiliary interactions (i.e., hidden target items).
This breakdown is intended to clearly reveal performance on hidden target items, which is essential for assessing generalization in real-world scenarios;
strong performance in the unobserved setting indicates the model’s ability to discover novel target items beyond those that users have already interacted~with.

\noindent\textbf{Hyperparameter Settings.}
Following~\cite{BCIPM,MuLe,HGIB,COPF,CRGCN,MBHGCN}, we set the embedding dimension d to $64$ and the batch size to $1024$ for all compared methods to ensure a fair comparison; in \proposed, each module uses $32$-dimensional embeddings, keeping the total dimension at $64$.
We use the Adam optimizer~\cite{kingma2017adammethodstochasticoptimization}, where the learning rate was tuned in $\{5 \cdot 10^{-4}, 10^{-4}\}$ and the weight decay was tuned in $\{0, 10^{-6}\}$.
All hyperparameters are tuned via grid search on the validation set.
For BOAR, the adversarial loss weight $\lambda_{\mathrm{adv}}$ was fixed to $0.01$, while the remaining loss weights $\lambda_{\mathrm{dense}}$ and $\lambda_\mathrm{refine}$ were tuned in the range of $[0.1, 1.0]$.
The contrastive temperatures were tuned in the range of $[0.1, 5.0]$.
The number of GNN layers was fixed to $2$ for all graph encoders, and the LSH projection dimension $d'$ was set to $32$, and the candidate retrieval size $n$ was tuned in $\{10, 20, 30, 40, 50\}$.
The Gumbel-Softmax temperature was fixed to $0.2$.
For baseline methods, we closely follow the hyperparameter search ranges reported in the original papers.


\subsection{Results and Analysis}
\subsubsection{\textbf{Overall Performance Comparison.}}
Table~\ref{tab:overall_performance} reports the performance of all methods 
under three evaluation settings: \textit{General}, \textit{Unobserved}, and \textit{Observed}.
\proposed achieves the best performance in the general and unobserved settings across all datasets and metrics, with particularly large gains in the auxiliary-unobserved setting.
In the general setting, \proposed consistently outperforms all baselines, achieving improvements of up to 7.82\% in HR@10 and 15.00\% in NDCG@10 over the best baseline.
Notably, in the unobserved setting, \proposed achieves substantially larger margins, with maximum improvements of 44.2\% in HR@10 and 36.7\% in NDCG@10.

In the observed setting, BOAR achieves the best or second-best performance across datasets; on datasets where HGIB leads, HGIB's exclusive focus on pruning-based rewiring allows it to concentrate full capacity on the observed environment.
This reflects a design trade-off: BOAR distributes model capacity across both environments, gaining substantially on unobserved items at a modest cost in the observed setting, and ultimately achieving the best overall performance in the general setting.
Furthermore, SHaRe consistently underperforms BOAR, confirming that conventional similarity-based rewiring may amplify MBR imperfections rather than resolving them.

\begin{figure}[t]
    \centering
    \includegraphics[width=\linewidth]{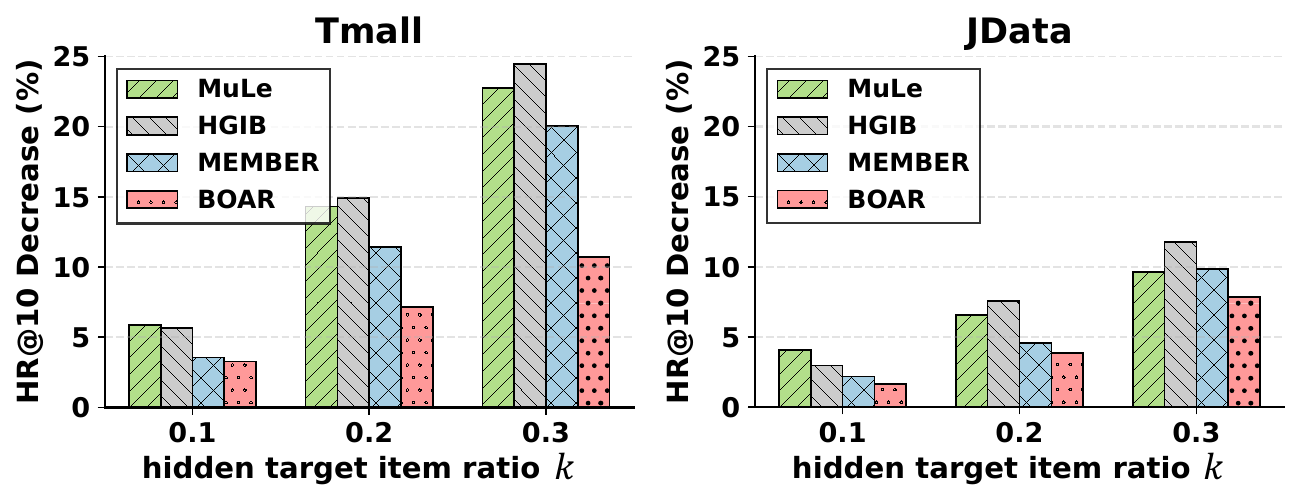}
    \caption{Performance drop in the hidden-preference test.}
    \label{fig:hidden_ratio}
    \vspace{-0.2cm}
\end{figure}

\begin{figure}[t]
    \centering
    \includegraphics[width=\linewidth]{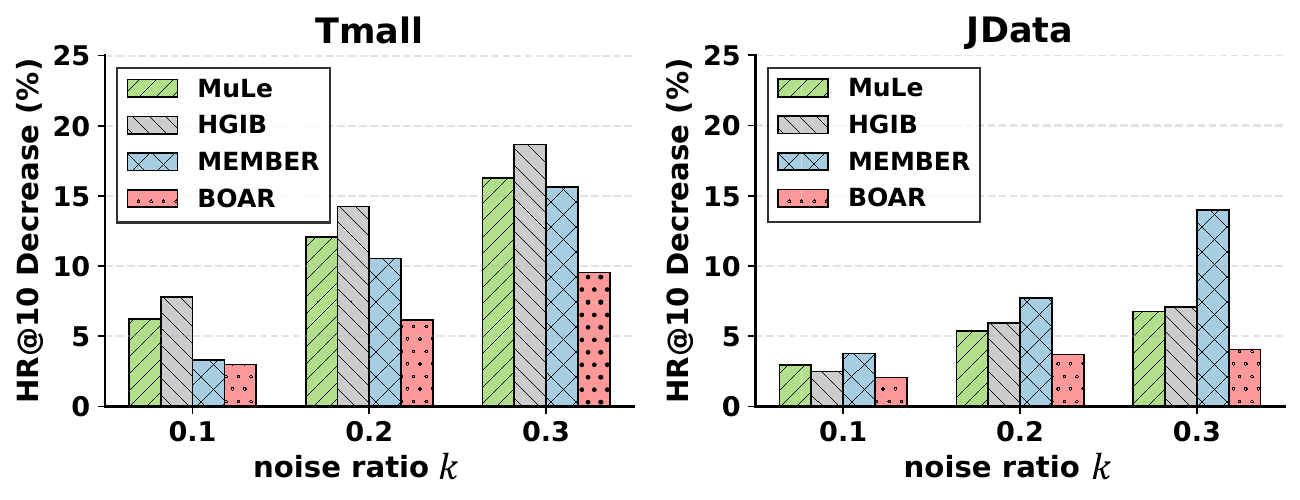}
    \caption{Performance drop in the noise-injection test.}
    \label{fig:noise_ratio}
    \vspace{-0.4cm}
\end{figure}

\subsubsection{\textbf{Stress-Testing of Graph Rewiring Modules.}}
\label{subsub:robust}
To verify that each rewiring module operates as intended, we design two controlled settings that directly examine the accuracy of densification and pruning:
(i) \emph{hidden preference test}, in which hidden target items are increased by removing a fraction $k$ of auxiliary interactions overlapping with target-positive pairs, placing greater demand on the densification of $f_0$; and 
(ii) \emph{noise injection test}, in which auxiliary noise is increased by injecting false auxiliary interactions into a fraction $k$ of target-negative pairs, placing greater demand on the pruning of $f_1$.
A method with accurate rewiring should exhibit minimal performance degradation as $k$ increases, since correct densification recovers the removed signal and correct pruning suppresses the injected noise.\footnote{Because interaction logs do not contain ground-truth purchase labels for every edge, direct edge-level evaluation is infeasible.
Therefore, we use hidden preference discovery and noise suppression as proxy tests for graph rewiring accuracy, reflecting the respective goals of densification and pruning.}

In the \textbf{hidden-preference test (Figure~\ref{fig:hidden_ratio})}, as $k$ increases, all methods degrade, indicating poor generalization when target-positive items lack auxiliary evidence.
\proposed shows the smallest performance drop across $k$, reflecting its ability to recover hidden preferences beyond observed auxiliary relations via debiased densification.
In the \textbf{noise-injection test (Figure~\ref{fig:noise_ratio})}, injecting noisy auxiliary interactions degrades all methods, highlighting the vulnerability of auxiliary-driven message passing.
\proposed exhibits the least degradation as $k$ increases, owing to its target-guided refinement that suppresses noise propagation in auxiliary signals.
These observations collectively show the effectiveness of \proposed in robust target prediction beyond auxiliary behavior imperfections, consistent with Taobao results in Figure~\ref{fig:figure1} (right).

\begin{table}[t]
  \centering
  \caption{Ablation study of BOAR on HR@10.}
  \label{tab:ablation}
  \begingroup
  \setlength{\tabcolsep}{5pt}
  \renewcommand{\arraystretch}{1.2}
  \resizebox{\columnwidth}{!}{%
    \begin{tabular}{l | ccc | ccc}
      \toprule
      \textbf{Setting} & \multicolumn{3}{c|}{\textbf{Unobserved}} & \multicolumn{3}{c}{\textbf{Observed}} \\
      \midrule
      \textbf{Datasets} & \textbf{Tmall} & \textbf{Taobao} & \textbf{JData} & \textbf{Tmall} & \textbf{Taobao} & \textbf{JData} \\
      \midrule
      \textbf{BOAR} & \textbf{0.1367} & \textbf{0.0352} & \textbf{0.5438} & \textbf{0.4587} & \textbf{0.5697} & \textbf{0.7757} \\
      \midrule
      \textbf{Ablation on densification} & & & & & & \\
      \quad w/o $\mathcal{L}_{\mathrm{dense}}$ & 0.1081 & 0.0307 & 0.4768 & \underline{0.4566} & 0.5607 & 0.7740 \\
      \quad w/o $\mathcal{L}_{\mathrm{adv}}$   & 0.1328 & 0.0335 & 0.5320 & 0.4549 & 0.5592 & \underline{0.7753} \\
      \quad w/o SNIPS                          & 0.1214 & 0.0338 & 0.5352 & 0.4565 & \underline{0.5645} & 0.7730 \\
      \textbf{Ablation on refinement} & & & & & & \\
      \quad w/o TAGR                          & 0.1321 & 0.0342 & 0.5391 & 0.4497 & 0.5623 & 0.7691 \\
      \quad w/o $\mathcal{L}_{\mathrm{refine}}$ & \underline{0.1329} & \underline{0.0344} & 0.5256 & 0.4445 & 0.5445 & 0.7636 \\
      \quad w/o Stop-gradient                  & 0.1290 & 0.0334 & 0.5345 & 0.4395 & 0.5426 & 0.7587 \\
      \textbf{Ablation on assignment} & & & & & & \\
      \quad Hard assignment & 0.1315 & 0.0295 & \underline{0.5408} & 0.4559 & 0.5607 & 0.7738 \\
      \quad Learnable assignment & 0.0609 & 0.0202 & 0.4698 & 0.4360 & 0.5411 & 0.7689 \\
      \bottomrule
    \end{tabular}
  }
  \endgroup
\end{table}

\subsubsection{\textbf{Ablation Study.}}
\label{subsub:ablation}
We present a detailed ablation study on the densification ($f_0$) and refinement module ($f_1$).
Table~\ref{tab:ablation} reports the ablation results under both settings, with $f_0$ focusing on the unobserved setting and $f_1$ on the observed setting.
For $f_0$, removing $\mathcal{L}_{\mathrm{dense}}$ leads to a large performance drop, confirming that hidden-preference mining is the primary source of gains in the unobserved environment.
In addition, incorporating both $\mathcal{L}_{\mathrm{adv}}$ and SNIPS yields the best performance, indicating a strong synergy between debiasing at the representation and prediction levels.
For $f_1$, replacing the target-guided auxiliary graph refiner (TAGR) with the standard LGCN encoder without refinement consistently degrades performance.
Moreover, $\mathcal{L}_{\mathrm{refine}}$, together with the stop-gradient operation that treats target embeddings as a fixed anchor, proves effective, supporting the validity of our design.

Lastly, we validate the modular integration design: replacing soft assignment with hard assignment degrades performance, as exclusive assignment prevents each instance from leveraging complementary signals across both modules; substituting a learnable assignment\footnote{
    We replace $[\hat{p}_{ui},\, 1-\hat{p}_{ui}]$ with 
    $[g_\theta(u,i),\, 1-g_\theta(u,i)] = \mathrm{softmax}(\mathbf{W}_g
    [\mathbf{e}_u \| \mathbf{e}_i]+\mathbf{b}_g)$, where 
    $\mathbf{e}_u, \mathbf{e}_i \in \mathbf{E}$ are the raw initial embeddings 
    and $\mathbf{W}_g \in \mathbb{R}^{2 \times 2d}$ is jointly optimized 
    with the prediction objective.
} causes a far more severe drop, as it lacks a principled basis for module assignment, and thus jointly optimizing the assignment with the prediction objective leads to unstable training.

\subsubsection{\textbf{Hyperparameter Study.}}
We provide an analysis to guide hyperparameter selection for \proposed.
Here, we investigate three key hyperparameters, while the remaining ones (e.g., $\lambda_{\mathrm{adv}}$) are fixed.

\noindent\textbf{Loss-balancing weights.}
Figure~\ref{fig:loss_sensitivity} presents the performance of $f_0$ and $f_1$ under varying $\lambda_{\mathrm{dense}}$ and $\lambda_{\mathrm{refine}}$, respectively.
Overall, \proposed shows stable performance with respect to $\lambda_{\mathrm{dense}}$, while small values of $\lambda_{\mathrm{refine}}$ are more effective across datasets.
A sufficiently large $\lambda_{\mathrm{dense}}$ is necessary to effectively exploit hidden preference supervision, whereas excessively increasing $\lambda_{\mathrm{refine}}$ can degrade observed-item ranking due to over-regularization.
In practice, moderate values consistently yield strong and stable performance.

\noindent\textbf{Candidate Retrieval Size.}
In Figure~\ref{fig:hyperparameter_sensitivity2}, we analyze the sensitivity of \proposed to the candidate size $n$ in angular LSH-based retrieval.
Overall, \proposed shows limited sensitivity to $n$, with only minor performance variations across datasets, indicating that learnable edge selection effectively filters unreliable candidates even as $n$ increases.
When $n$ is too small, the effect is limited, as potential hidden target-positive candidates may be missed.

\begin{figure}[t]
    \centering
    \begin{subfigure}{\linewidth}
        \centering
        \includegraphics[width=\linewidth]{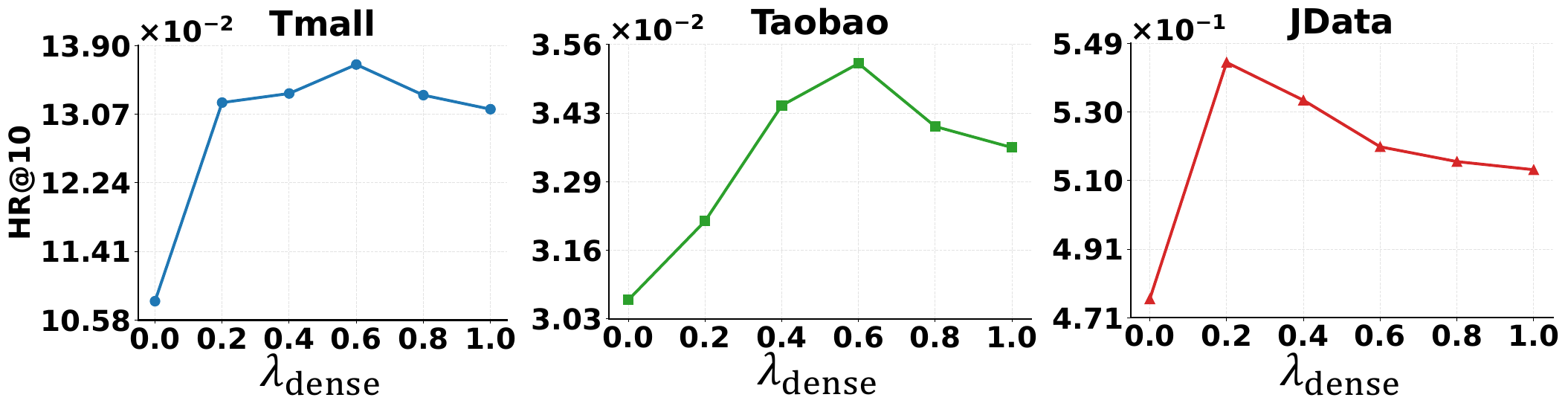}
        \caption{Effect of $\lambda_{\mathrm{dense}}$ on $f_0$ in the \textit{unobserved} setting.}
    \end{subfigure}
    \vspace{0.5em}
    \begin{subfigure}{\linewidth}
        \centering
        \includegraphics[width=\linewidth]{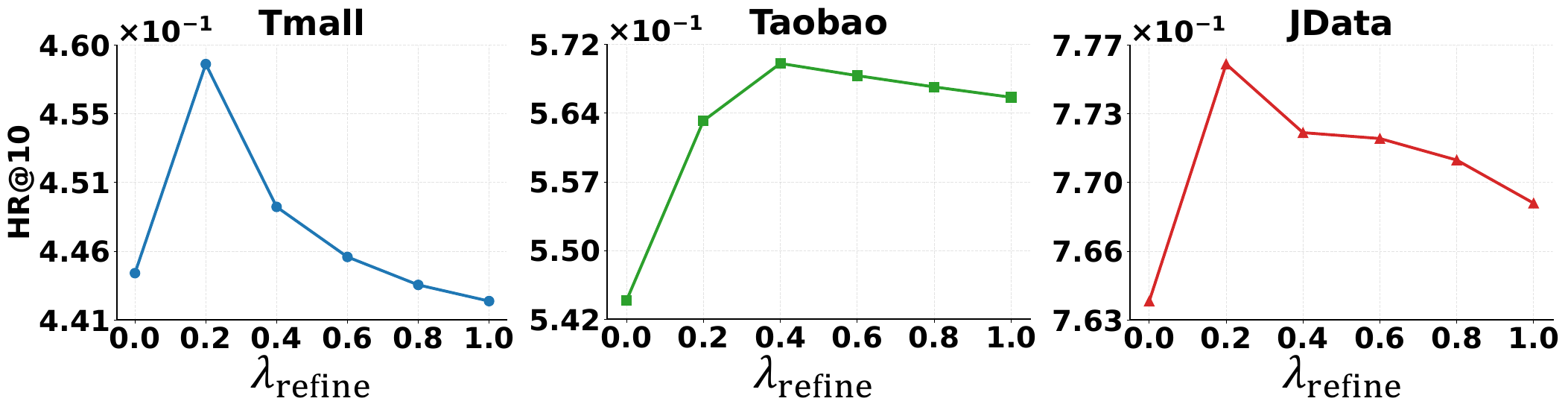}
        \caption{Effect of $\lambda_{\mathrm{refine}}$ on $f_1$ in the \textit{observed} setting.}
    \end{subfigure}
    \caption{Effects of loss-balancing hyperparameters.}
    \label{fig:loss_sensitivity}
\end{figure}

\begin{figure}[t]
    \centering
    \includegraphics[width=\linewidth]
    {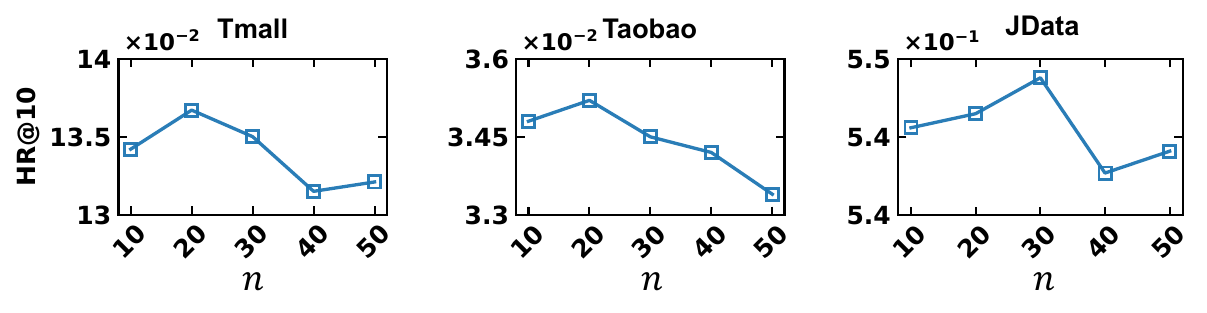}
    \caption{HR@10 under different candidate retrieval sizes $n$.}
    \label{fig:hyperparameter_sensitivity2}
    \vspace{-0.2cm}
\end{figure}

\begin{table}[h]
  \centering
  \caption{Efficiency analysis: comparison of per-epoch training time (sec), inference time (sec), and model size (in millions).
  Training time denotes the average time per epoch, and inference time denotes the time to generate recommendations for all users.}
  \begingroup
  \renewcommand{\arraystretch}{1.}
  \resizebox{0.95\linewidth}{!}{
  \begin{tabular}{llcccc}
    \toprule
    \textbf{Dataset} & \textbf{Metric} & \textbf{MuLe} & \textbf{HGIB} & \textbf{MEMBER} & \textbf{BOAR} \\
    \midrule
      & Train time / epoch & 81.47s & 86.46s & 164.85s & 57.93s \\
     \textbf{Tmall} & Inference time   & 12.52s & 11.24s & 12.82s & 8.05s \\
      & \#Parameters  & 3.44M    & 3.44M    & 3.44M    & 3.44M            \\
    \midrule
      & Train time / epoch  & 70.22s & 64.43s & 116.97s & 49.86s \\
     \textbf{Taobao} & Inference time & 21.69s & 19.66s & 22.84s & 16.99s \\
      & \#Parameters  & 5.65M    & 5.66M    & 5.65M    & 5.65M             \\
    \midrule
      & Train time / epoch   & 100.12s & 90.83s & 178.49s & 95.86s \\
     \textbf{JData} & Inference time & 5.92s  & 5.69s  & 6.22s  & 4.68s \\
      & \#Parameters  & 7.55M    & 7.56M    & 7.55M    & 7.56M            \\
    \bottomrule
  \end{tabular}}
  \endgroup
  \label{tab:inference_time}
\end{table}
\vspace{-1em}
\subsection{Complexity Analysis}
\label{app:complexity}
Table~\ref{tab:inference_time} compares the per-epoch training time and inference time, as well as the model size, of \proposed with state-of-the-art MBR methods~\cite{MuLe, HGIB, MEMBER}.
All experiments are conducted using PyTorch with CUDA on an RTX 5000 Ada GPU and an Intel Xeon Gold 6338 CPU.

\proposed introduces only a small number of additional parameters, mainly a few lightweight linear layers, and since the densification process is applied only during training, it achieves the fastest inference latency among all compared methods.
\proposed also achieves the fastest per-epoch training time on Tmall and Taobao, and remains comparable to HGIB on JData.
Overall, \proposed achieves strong performance with efficiency comparable to state-of-the-art methods, supporting its practical applicability.

\section{Conclusion}
In this work, we first highlight two fundamental challenges of auxiliary behaviors: missing auxiliary signals and unreliable auxiliary signals.
As a solution, we propose \proposed, an environment-conditioned framework that addresses missing and unreliable auxiliary signals through two complementary modules conditioned on auxiliary observability.
Extensive experiments show that \proposed achieves particularly strong gains on hidden target items without auxiliary behaviors, validating its effectiveness in capturing hidden preferences beyond observed auxiliary relations.
Future work may explore streaming settings with continuously evolving user behaviors and interactions.
\appendix
\section{Appendix}


\subsection{Supplementary Proofs and Method Details}
\label{app:supp_method}



\subsubsection{\textbf{On the Validity of Debiasing in \proposed}}
\label{app:ips-validity}
We provide a theoretical interpretation of two debiasing techniques underlying
\proposed: Inverse Propensity Scoring (IPS)~\cite{schnabel2016recommendations}
and adversarial learning~\cite{ben2010theory}.

\smallsection{IPS weighted BPR Loss (Eq.~\eqref{eq:dbpr}).}
\label{app:ips-theory}
Let $O_{ui}\in\{0,1\}$ denote whether the target-behavior pair $(u,i)$ is observable in the training data (e.g., due to platform exposure), with propensity $p_{\mathrm{obs}}(u,i)=\Pr(O_{ui}=1\mid u,i)$.
We assume target-positive signals are observed only when $O_{ui}=1$.
For an observed positive $(u,i)\in\mathcal{E}_{\mathrm{tar}}$, we sample $j\sim q(j\mid u)$ (independent of $O_{ui}$) and use $\ell_\theta(u,i,j)=-\log\sigma(s_{ui}-s_{uj})$.
Although $O_{ui}=1$ for all logged pairs in practice, $O_{ui}$ is treated as a Bernoulli random variable in the data-generating process, enabling importance weighting.

Define the ideal risk $\mathcal{L}_{\mathrm{ideal}}= \mathbb{E}[\ell_\theta(u,i,j)]$ and the IPS-weighted risk
$\mathcal{L}_{\mathrm{IPS}}=\mathbb{E}\!\left[\frac{O_{ui}}{\hat{p}_{\mathrm{obs}}(u,i)}\ell_\theta(u,i,j)\right]$.

\begin{theorem}[Unbiasedness under correct 
propensities~\cite{schnabel2016recommendations}]
\label{thm:ips}
Assume $p_{\mathrm{obs}}(u,i)>0$.
If $\hat{p}_{\mathrm{obs}}(u,i)=p_{\mathrm{obs}}(u,i)$ almost surely, then $\mathrm{Bias}[\mathcal{L}_{\mathrm{IPS}}] =|\mathcal{L}_{\mathrm{IPS}}-\mathcal{L}_{\mathrm{ideal}}|
=0$.
\end{theorem}

\begin{proof}
Since $\ell_\theta(u,i,j)$ is deterministic given $(u,i,j)$,
\[
\mathcal{L}_{\mathrm{IPS}}
=\mathbb{E}_{u,i,j}\!\left[\ell_\theta(u,i,j)\cdot
\mathbb{E}\!\left[\left.
\tfrac{O_{ui}}{\hat{p}_{\mathrm{obs}}(u,i)}
\right|u,i\right]\right].
\]
Since $\mathbb{E}[O_{ui}\mid u,i]=p_{\mathrm{obs}}(u,i)=
\hat{p}_{\mathrm{obs}}(u,i)$, the inner expectation is $1$,
yielding $\mathcal{L}_{\mathrm{IPS}}=
\mathcal{L}_{\mathrm{ideal}}$.
\end{proof}

\paragraph{Connection to Eq.~\eqref{eq:dbpr}.}
Eq.~\eqref{eq:dbpr} approximates $\mathcal{L}_{\mathrm{IPS}}$ 
via Monte-Carlo over logged positives.
Following~\cite{li2023stabledr, schnabel2016recommendations,Zhu0WW25}, we estimate the pair-level propensity $p_{\mathrm{obs}}(u,i)$ via a lightweight binary classifier $\hat{p}_{ui}$ (Appendix~\ref{app:propensity}), trained to predict whether a user--item pair is auxiliary-observed. We set $\omega_{ui} \propto 1/\hat{p}_{ui}$ (SNIPS), up-weighting pairs with low auxiliary-observation propensity and down-weighting those with high propensity to correct observation bias.

\smallsection{Auxiliary Popularity Adversarial Learning 
(Eq.~\eqref{eq:adv}).}
Let $\mathcal{D}_{\mathrm{above}}$ and 
$\mathcal{D}_{\mathrm{below}}$ denote the item distributions 
over items whose auxiliary interaction count exceeds ($y_i=1$) or falls below ($y_i=0$) the median (Sec.~4.2.1).
Let $z=g_\psi(i)\in\mathcal{Z}$ be the learned item 
representation, with induced distributions
$P_\psi:=g_\psi\#\mathcal{D}_{\mathrm{above}}$ and 
$Q_\psi:=g_\psi\#\mathcal{D}_{\mathrm{below}}$.
Let $\mathcal{H}_d$ be binary classifiers 
$h:\mathcal{Z}\to\{0,1\}$ where $h(z){=}1$ predicts above-median 
and $h(z){=}0$ predicts below-median.

\begin{lemma}[$\mathcal{H}$-divergence~\cite{ben2010theory}]
\label{lem:hdiv}
\[
d_{\mathcal{H}_d}(P_\psi,Q_\psi)
:=2\sup_{h\in\mathcal{H}_d}
\left|\Pr_{z\sim P_\psi}[h(z){=}1]
-\Pr_{z\sim Q_\psi}[h(z){=}1]\right|.
\]
With $\varepsilon_d^*:=\min_{h}\varepsilon_d(h)$ where
$\varepsilon_d(h):=\frac{1}{2}\Pr_{P_\psi}[h{=}0]
+\frac{1}{2}\Pr_{Q_\psi}[h{=}1]$,
and $\mathcal{H}_d$ complement-closed,
$d_{\mathcal{H}_d}(P_\psi,Q_\psi)=2(1-2\varepsilon_d^*)$.
Hence $\varepsilon_d^*\to\tfrac{1}{2}$ implies 
$d_{\mathcal{H}_d}(P_\psi,Q_\psi)\to 0$.
\end{lemma}

\begin{proof}
For any $h\in\mathcal{H}_d$,
$\Pr_{P_\psi}[h{=}1]-\Pr_{Q_\psi}[h{=}1]
=1-2\varepsilon_d(h)$.
Since $\mathcal{H}_d$ is complement-closed,
$\sup_h|\Pr_{P_\psi}[h{=}1]-\Pr_{Q_\psi}[h{=}1]|
=1-2\varepsilon_d^*$. Multiplying by $2$ yields the claim.
\end{proof}

\begin{corollary}[Why both GRL and IPS are required]
\label{cor:grl-ips}
If the GRL-based adversarial objective drives 
$\varepsilon_d^*\approx\tfrac{1}{2}$, then by 
Lemma~\ref{lem:hdiv}, $d_{\mathcal{H}_d}(P_\psi,Q_\psi)
\approx 0$, meaning $g_\psi$ suppresses popularity-driven observation bias from item representations in $\mathbf{E}^{\mathrm{glo}}$.
This encourages LSH-based candidate retrieval (Sec.~\ref{subsubsec:thpm}) to surface preference-aligned hidden items rather than popularity-similar ones; these candidates are incorporated into $\tilde{\mathbf{A}}_{\mathrm{tar}}$ and propagated to $\mathbf{E}^{\mathrm{tar}}$ via LGCN, supporting reliable recovery of hidden preferences at the representation level.

However, even with suppressed popularity-driven observation bias in representations, optimizing the ranking objective on logged positives can still be dominated by frequently auxiliary-observed items, since they appear more often in $\mathcal{E}_{\mathrm{tar}}$ regardless of their representation.
The propensity-weighted objective in Theorem~\ref{thm:ips} corrects this prediction-level bias, while SNIPS provides variance-stabilized approximations.
Therefore, removing either component leaves (i) popularity-driven observation bias in $\mathbf{E}^{\mathrm{glo}}$, degrading candidate quality for hidden items (without GRL), or (ii) observation bias in the ranking objective toward frequently auxiliary-observed items (without IPS), both of which hinder reliable recovery of hidden preferences.
\end{corollary}

\subsubsection{\textbf{Propensity Computation}}
\label{app:propensity}
Following~\cite{schnabel2016recommendations, li2023stabledr, Zhu0WW25}, 
we estimate $\hat{p}_{ui} = \sigma({\mathbf{e}^{\mathrm{aux}}_u}^\top \mathbf{e}^{\mathrm{aux}}_i)$ via a lightweight binary classifier, where $\mathbf{E}^{\mathrm{aux}} = \mathrm{LGCN}(\mathbf{A}_{\mathrm{aux}}, \mathbf{E})$ and $\mathbf{A}_{\mathrm{aux}}$ is constructed from $\mathcal{E}_{\mathrm{aux}}$.
The classifier is trained via binary cross-entropy with $E_{ui}$ (Sec.~\ref{subsec:env_modeling}) as supervision labels, where $\mathcal{D}$ is constructed by sampling positives from $\mathcal{E}_{\mathrm{aux}}$ and, for each user $u$, randomly sampling auxiliary-unobserved items as negatives at a 1:1 ratio:
\begin{equation}
\mathcal{L}_{\mathrm{BCE}} = 
- \sum_{(u,i)\in \mathcal{D}} 
\bigl[E_{ui}\log \hat{p}_{ui} + (1-E_{ui})\log(1-\hat{p}_{ui})\bigr].
\end{equation}
We then apply self-normalized inverse propensity scoring:
\begin{equation}
\omega_{ui} = \frac{\hat{p}_{ui}^{-1}}
{\sum_{(u',i')\in\mathcal{E}_{\mathrm{tar}}} \hat{p}_{u'i'}^{-1}},
\end{equation}
where $\hat{p}_{ui}$ is clipped to $\max(\hat{p}_{ui}, 10^{-5})$ for numerical stability.

\begin{acks}
This work was supported by a Korea University Grant, IITP grants funded by the Korea government (MSIT): the ICT Creative Consilience Program (IITP-2026-RS-2020-II201819) and the Artificial Intelligence Star Fellowship (IITP-2026-RS-2025-02304828).
This work was also supported by NRF grants funded by the MSIT (RS-2026-25486220) and by the Basic Science Research Program funded by the Ministry of Education (NRF-2021R1A6A1A03045425).
\end{acks}

\section*{GenAI Usage Disclosure}
The authors employed LLM tools in a limited capacity, namely polishing grammar in the written manuscript and assisting with code debugging. All such outputs were verified and revised by the authors.
The core research process, including problem formulation and model design, was conducted entirely without LLM assistance.

\bibliographystyle{ACM-Reference-Format}
\vspace{-1em}
\bibliography{reference}
\end{document}